\documentclass{article}
\usepackage[final, nonatbib]{neurips_2021}

\usepackage{times}
\usepackage{graphicx}
\usepackage{hyperref}
\usepackage{amsfonts}
\usepackage{amsmath}
\usepackage{amsthm}
\usepackage{subcaption}
\usepackage{algorithm}
\usepackage{mathtools}
\usepackage{algpseudocode}
\usepackage{caption}
\usepackage{setspace}
\usepackage{xcolor}

\newcommand{\eg}{{\it e.g.}}
\newcommand{\ie}{{\it i.e.}}

\newcommand{\BA}{\begin{array}}
\newcommand{\EA}{\end{array}}

\newcommand{\BIT}{\begin{itemize}}
\newcommand{\EIT}{\end{itemize}}

\newcommand{\reals}{{\mathbb{R}}} %
\newcommand{\nats}{{\mathbb{N}}} %

\newcommand{\Expect}{\mathbb{E}}

\newcommand{\argmax}{\mathop{\rm argmax}}

\newcommand{\bayesregret}{\mathcal{BR}_\phi}
\newcommand{\regret}{\mathcal{R}}

\newcommand{\Bc}{\mathcal{B}}
\newcommand{\Tc}{\mathcal{T}}
\newcommand{\Fc}{\mathcal{F}}
\newcommand{\Sc}{\mathcal{S}}
\newcommand{\Ac}{\mathcal{A}}

\newcommand{\Lc}{\mathcal{L}}

\newcommand{\cevs}{\mathcal{Q}^t}
\newcommand{\episodes}{N}
\newcommand{\epregret}{\Phi}
\newcommand{\lse}[1]{\log \sum_{#1} \exp}
\DeclarePairedDelimiter{\ceil}{\lceil}{\rceil}

\newtheorem{assumption}{Assumption}
\newtheorem{lemma}{Lemma}
\newtheorem{theorem}{Theorem}

\begin{document}

\title{Variational Bayesian Reinforcement Learning with Regret Bounds}
\author{%
Brendan O'Donoghue\\
DeepMind, UK \\
 \texttt{bodonoghue@google.com}
}

\maketitle

\begin{abstract}
In reinforcement learning the Q-values summarize the expected future rewards
that the agent will attain. However, they cannot capture the epistemic
uncertainty about those rewards. In this work we derive a new Bellman operator
with associated fixed point we call the `knowledge values'. These K-values
compress both the expected future rewards and the epistemic uncertainty into a
single value, so that high uncertainty, high reward, or both, can yield high
K-values. The key principle is to endow the agent with a risk-seeking utility
function that is carefully tuned to balance exploration and exploitation.  When
the agent follows a Boltzmann policy over the K-values it yields a Bayes regret
bound of $\tilde O(L \sqrt{S A T})$, where $L$ is the time horizon, $S$ is the
total number of states, $A$ is the number of actions, and $T$ is the number of
elapsed timesteps. We show deep connections of this approach to the soft-max
and maximum-entropy strands of research in reinforcement learning.

\end{abstract}

\section{Introduction and related work}

In reinforcement learning (RL) an agent interacts with an environment in an episodic
manner and attempts to maximize its return \cite{sutton:book,
puterman2014markov}. In this work the environment is a
Markov decision process (MDP) and we consider the Bayesian case
where the agent has some prior information and as it gathers data it updates
its posterior beliefs about the environment.  In this setting the agent is faced
with the choice of visiting well understood states or exploring the
environment to determine the value of other states which might lead to a higher
return.  This trade-off is called the \emph{exploration-exploitation} dilemma.
One way to measure how well an agent balances this trade-off is a quantity called
\emph{regret}, which measures how sub-optimal the rewards the agent has received
are so far, relative to the (unknown) optimal policy \cite{cesa2006prediction}.
In the Bayesian case the natural quantity to consider is the Bayes regret, which
is the expected regret under the agents prior information
\cite{ghavamzadeh2015bayesian}.

The optimal Bayesian policy can be formulated using \emph{belief
states}, but this is believed to be intractable for all but small problems
\cite{ghavamzadeh2015bayesian}.
Approximations to the optimal Bayesian policy exist, one of the most successful
being Thompson sampling %
\cite{strens2000bayesian, thompson1933likelihood} wherein the agent samples
from the posterior over value functions and acts greedily with respect to that sample
\cite{osband2013more, osband2014generalization, lipton2016efficient,
osband2016posterior}. It can be shown that this
strategy yields both Bayesian and frequentist regret bounds under certain
assumptions \cite{agrawal2017near}. In practice, maintaining a posterior over
value functions is intractable, and so instead the agent maintains the posterior
over MDPs, and at each episode an MDP is sampled from this posterior, the value
function for that sample is computed, and the policy acts greedily with respect to
that value function. Due to the repeated sampling and computing of value functions
this is practical only for small problems, though attempts have been made
to extend it \cite{osband2016deep, o2018uncertainty}.

Bayesian algorithms have the advantage of being able to incorporate prior
information about the problem and, as we show in the numerical experiments, they
tend to perform better than non-Bayesian approaches in practice
\cite{ghavamzadeh2015bayesian, sorg2010variance, russo2018tutorial}.  Although
typically Bayes regret bounds hold for any prior that satisfies the assumptions,
the requirement that the prior over the MDP is known in advance is a
disadvantage for Bayesian methods.  One common concern is about performance
degradation when the prior is misspecified. In this case it can be shown that
the regret increases by a multiplicative factor related to the Radon-Nikodym
derivative of the true prior with respect to the assumed prior \cite[\S
3.1]{russo2014learning}. In other words, a Bayesian algorithm with sub-linear
Bayes regret operating under a misspecified prior will still have sub-linear
regret so long as the true prior is absolutely continuous with respect to the
misspecified prior. Moreover, for any algorithm that satisfies a Bayes regret
bound it is straightforward to derive a high-probability regret bound for any
family of MDPs that has support under the prior, in a sense translating Bayes
regret into frequentist regret; see  \cite[\S 3.1]{russo2014learning},
\cite[Appendix A]{osband2013more} for details.

In this work we endow an agent with a particular \emph{epistemic risk-seeking}
utility function, where `epistemic risk' refers to the Bayesian uncertainty that
the agent has about the optimal value function of the MDP.  In the context of
RL, acting so as to maximize a risk-seeking utility function which assigns
higher values to more uncertain actions is a form of \emph{optimism in the face
of uncertainty}, a well-known heuristic to encourage exploration
\cite{jaksch2010near, auer2002finite}.  Any increasing convex function could be
used as a risk-seeking utility, however, only the exponential utility function
has a decomposition property which is required to derive a Bellman recursion
\cite{abbas2007invariant, pfanzag1959general, howard1967value, raiffa68dec}. We
call the fixed point of this Bellman operator the `K-values' for
\emph{knowledge} since they compress the expected downstream reward and the
downstream epistemic uncertainty at any state-action into a single quantity. A
high K-value captures the fact that the state-action has a high expected Q-value
or high uncertainty, or both.  Following a Boltzmann policy over the K-values
yields a practical algorithm that we call `K-learning' which attains a Bayes
regret upper bounded by $\tilde O(L \sqrt{S A T})$ \footnote{Previous versions
of this manuscript had a $L^{3/2}$ dependency. This was because $S$ was
interpreted as the number of states per-timestep, but not clearly defined that
way. This version corrects this and makes clear that $S$ now refers the
\emph{total} number of states.} where $L$ is the time horizon, $S$ is the total
number of states, $A$ is the number of actions per state, and $T$ is the number
of elapsed timesteps \cite{cesa2017boltzmann}. This regret bound matches the
best known bound for Thompson sampling up to log factors \cite{osband2013more}
and is within a factor of $\sqrt{L}$ of the known information theoretic lower
bound of $\Omega(\sqrt{LSAT})$ \cite[Appendix D]{jin2018q}.

The update rule we derive is similar to that used in `soft' Q-learning
(so-called since the `hard' max is replaced with a soft-max)
\cite{azar2012dynamic, glearning, haarnoja17, nachum2017bridging,
maginnis2017qlearn}.  These approaches are very closely related to maximum
entropy reinforcement learning techniques which add an entropy regularization
`bonus' to prevent early convergence to deterministic policies and thereby
heuristically encourage exploration \cite{williams1991function,
ziebart2010modeling, mnih2016asynchronous, o2016pgq, map2018,
levine2018rlasinf}.  In our work the soft-max operator and entropy
regularization arise naturally from the view of the agent as maximizing a
risk-seeking exponential utility.  Furthermore, in contrast to these other
approaches, the entropy regularization is not a fixed hyper-parameter but
something we explicitly control (or optimize for) in order to carefully
trade-off exploration and exploitation.

The algorithm we derive in this work is model-based, \ie, requires estimating
the full transition function for each state.  There is a parallel strand of work
deriving regret and complexity bounds for \emph{model-free} algorithms,
primarily based on extensions of Q-learning \cite{jin2018q, zhang2020almost,
li2020sample}.  We do not make a detailed comparison between the two approaches
here other than to highlight the advantage that model-free algorithms have both
in terms of storage and in computational requirements. On the other hand, in the
numerical experiments model-based approaches tend to outperform the model-free
algorithms. We conjecture that an online, model-free version of K-learning with
similar regret guarantees can be derived using tools developed by the model-free
community. We leave exploring this to future work.

\subsection{Summary of main results}
\BIT
\item We consider an agent endowed an epistemic risk-seeking utility
function and derive a new optimistic Bellman operator that incorporates the
`value' from epistemic uncertainty about the MDP. The new operator
replaces the usual max operator with a soft-max and it incorporates a `bonus'
that depends on state-action visitation.  In the limit of zero uncertainty the
new operator reduces to the standard optimal Bellman operator.
\item At each episode we solve the optimistic Bellman equation for the
`K-values' which represent the utility of a particular
state and action. If the agent follows a Boltzmann policy over the K-values with a
carefully chosen temperature schedule then it will enjoy a sub-linear Bayes
regret bound.
\item  To the best of our knowledge this is the first work to show that soft-max
operators and maximum entropy policies in RL can provably yield good performance
as measured by Bayes regret.  Similarly, we believe this is the first result
deriving a Bayes regret bound for a Boltzmann policy in RL.  This puts
maximum entropy, soft-max operators, and Boltzmann exploration in a principled
Bayesian context and shows that they are naturally derived from endowing the
agent with an exponential utility function.
\EIT

\section{Markov decision processes}

In a Markov decision process (MDP) an agent interacts with an environment in a
series of episodes and attempts to maximize the cumulative reward.  We model
the environment as a finite state-action, time-inhomogeneous MDP given by the
tuple $\mathcal{M} = \{\Sc, \Ac, R, P, L, \rho\}$, where $\Sc$ is the
state-space, $\Ac$ is the action-space, $R_l(s, a)$ is a probability
distribution over the rewards received by the agent at state $s$ taking action
$a$ at timestep $l$, $P_l(s^\prime\mid s, a) \in [0,1]$ is the probability the
agent will transition to state $s^\prime$ after taking action $a$ in state $s$
at timestep $l$, $L \in \nats$ is the episode length, and $\rho$ is the initial
state distribution.  We assume that the state space can be decomposed layerwise
as $\Sc = \Sc_1 \cup \Sc_2 \cup \ldots \cup \Sc_L$ and it has cardinality
$|\Sc|=\sum_{l=1}^L |\Sc_l| = S$, and the cardinality of the action space is
$|\Ac| = A$.  Concretely, the initial state $s_1 \in \Sc_1$ of the agent is
sampled from $\rho$, then for timesteps $l = 1, \ldots, L$ the agent is in
state $s_l \in \Sc_l$, selects action $a_l \in \Ac$, receives reward $r_l \sim
R_l(s_l, a_l)$ with mean $\mu_l(s_l,a_l) \in \reals$ and transitions to the
next state $s_{l+1} \in \Sc_{l+1}$ with probability $P_l(s_{l+1} \mid s_l,
a_l)$. After timestep $L$ the episode terminates and the state is reset.  We
assume that at the beginning of learning the agent does not know the reward or
transition probabilities and must learn about them by interacting with the
environment. We consider the Bayesian case in which the mean reward $\mu$ and
the transition probabilities $P$ are sampled from a known prior $\phi$. We
assume that the agent knows $S$, $A$, $L$, and the reward noise distribution.

An agent following policy $\pi_l : \Sc_l \times \Ac \rightarrow [0,1]$
at state $s \in \Sc_l$ at time $l$ selects action $a$ with probability
$\pi_l(s, a)$.
The Bellman equation relates the value
of actions taken at the current timestep to future returns through the
\emph{Q-values} and the associated \emph{value function} \cite{bellman}, which for
policy $\pi$ are denoted $Q_l^\pi \in \reals^{|\Sc_l| \times A}$ and $V_l^\pi \in
\reals^{|\Sc_l|}$ for $l=1, \ldots,L+1$, and satisfy
\begin{equation}
\label{e-bell-pol}
  Q^\pi_l = \Tc_l^\pi Q^\pi_{l+1}, \quad V_l^\pi(s) = \sum_{a \in \Ac} \pi_l(s, a)
Q_l^\pi(s, a),
\end{equation}
for $l=1, \ldots, L$ where $Q_{L+1} \equiv 0$ and where the Bellman operator for
policy $\pi$ at step $l$ is defined as
\[
  (\Tc_l^\pi Q^\pi_{l+1})(s, a) \coloneqq \mu_l(s, a) + \sum_{s^\prime \in \Sc_{l+1}}P_l(s^\prime \mid s,
a) \sum_{a^\prime \in \Ac} \pi_l(s^\prime, a^\prime) Q_{l+1}^\pi(s^\prime, a^\prime).
\]
The expected performance of policy $\pi$ is denoted $J^\pi = \Expect_{s \sim \rho}
V_1^\pi(s)$.
An optimal policy satisfies $\pi^\star
\in \argmax_\pi J^\pi$ and induces associated optimal
Q-values and value function given by
\begin{equation}
\label{e-bellman_q}
  Q^\star_l = \Tc_l^\star Q^\star_{l+1}, \quad V_l^\star(s) = \max_{a} Q_l^\star(s, a).
\end{equation}
for $l=1, \ldots, L$, where $Q^\star_{L+1} \equiv 0$ and
where the optimal Bellman operator is defined at step $l$ as
\begin{equation}
\label{e-bellman_opt_op}
  (\Tc_l^\star Q^\star_{l+1})(s, a) \coloneqq \mu_l(s, a) + \sum_{s^\prime \in \Sc_{l+1}}P_l(s^\prime
\mid s, a) \max_{a^\prime} Q_{l+1}^\star(s^\prime, a^\prime).
\end{equation}

\subsection{Regret}
If the mean reward $\mu$ and transition function $P$ are known exactly then (in
principle) we
could solve \eqref{e-bellman_q} via dynamic programming
\cite{bertsekas2005dynamic}.  However, in practice these are not known and so
the agent must gather data by interacting with the environment over a series of
episodes. The key trade-off is the \emph{exploration-exploitation} dilemma,
whereby an agent must take possibly suboptimal actions in order to learn about
the MDP.  Here we are interested in the \emph{regret}
up to time $T$, which is how sub-optimal the agent's policy has
been so far. The regret for an algorithm producing policies $\pi^t$,
$t=1, \ldots, N$ executing on MDP $\mathcal{M}$ is defined as
\[
\regret_{\mathcal{M}}(T) \coloneqq \sum_{t=1}^{\episodes}
\Expect_{s \sim \rho} (V_1^\star(s) - V_1^{\pi^t}(s)),
\]
where $\episodes \coloneqq \ceil{T/L}$ is the number of elapsed episodes. %
In this manuscript we take the case where $\mathcal{M}$
is sampled from a known prior $\phi$ and we want to minimize the expected regret of
our algorithm under that prior distribution. This is referred to as the
\emph{Bayes regret}:
\begin{equation}
\label{e-bayesregret1}
\bayesregret(T) \coloneqq \Expect_{\mathcal{M} \sim \phi} \regret_{\mathcal{M}}(T) = \Expect_{\mathcal{M}\sim \phi} \sum_{t=1}^{\episodes}
\Expect_{s \sim \rho} (V_1^\star(s) - V_1^{\pi^t}(s)).
\end{equation}
In the Bayesian view of the RL problem the quantities $\mu$ and $P$ are random
variables, and consequently the optimal Q-values $Q^\star$, policy $\pi^\star$,
and value function $V^\star$ are also random variables that must be learned about
by gathering data from the environment. We shall denote by $\Fc_t$
the sigma-algebra generated by all the history \emph{before} episode $t$ where
$\Fc_1 = \emptyset$ and we shall use $\Expect^t$ to denote $\Expect(~\cdot \mid
\Fc_t)$, the expectation conditioned on $\Fc_t$. For example, with this notation
$\Expect^t Q^\star$ denotes the expected optimal Q-values under the posterior
before the start of episode $t$.

\section{K-learning}
Now we present Knowledge Learning (K-learning), a Bayesian RL algorithm that
satisfies a sub-linear Bayes regret guarantee. In standard dynamic programming
the Q-values are the unique fixed point of the Bellman equation, and they
summarize the expected future reward when following a particular policy.
However, standard Q-learning is not able to incorporate any of the uncertainty
about future rewards or transitions. In this work we develop a new Bellman
operator with associated fixed point we call the `K-values' which represent
\emph{both} the expected future rewards and the uncertainty about those rewards.
These two quantities are compressed into a single value by the use of an
exponential risk-seeking utility function, which is tuned to trade-off
exploration and exploitation.  In this section we develop the intuition behind
the approach and defer all proofs to the appendix.  We begin with the main
assumption that we require for the analysis (this assumption is standard, see,
\eg, \cite{osband2016posterior}).

\begin{assumption}
\label{ass1}
The mean rewards are bounded in $[0,1]$ almost surely with independent
priors, the reward noise is additive $\sigma$-sub-Gaussian, and the prior over
transition functions is independent Dirichlet.
\end{assumption}

\subsection{Utility functions and the certainty equivalent value}
A utility function $u: \reals \rightarrow \reals$ measures an agents preferences
over outcomes \cite{von2007theory}. If $u(x) > u(y)$ for some $x, y \in \reals$
then the agent prefers $x$ to $y$, since it derives more utility from $x$ than
from $y$.
If $u$ is convex then it is referred to as \emph{risk-seeking}, since $\Expect
u(X) \geq u(\Expect X)$ for random variable $X$ due to Jensen's inequality.
The particular utility function we shall use is the exponential
utility $u(x) \coloneqq \tau(\exp(x / \tau) - 1)$ for some $\tau \geq 0$.  The
\emph{certainty equivalent value} of a random variable under utility $u$
measures how much guaranteed payoff is equivalent to a random payoff, and for
$Q_l^\star(s,a)$ under the exponential utility is given by
\begin{equation}
\label{e-certequiv}
\cevs_l(s,a) \coloneqq u^{-1}(\Expect^t u(Q_l^\star(s,a)) = \tau \log \Expect^t \exp (Q_l^\star(s,a)/\tau).
\end{equation}
This is the key quantity we use to summarize the expected value and the epistemic
uncertainty into a single value.  As an example, consider a stochastic
multi-armed bandit (\ie, an MDP with $L=1$ and $S=1$) where the prior over the
rewards and the reward noise are independent Gaussian distributions. At round
$t$ the posterior over $Q^\star(a)$ is given by $\mathcal{N}(\mu^t_a,
(\sigma^t_a)^2)$ for some $\mu^t_a$ and $\sigma^t_a$ for each action $a$,
due to the conjugacy of the prior and the likelihood. In this case the certainty
equivalent value can be calculated using the Gaussian cumulant generating
function, and is given by $\cevs(a) = \mu^t_a + (1/2) (\sigma^t_a)^2 / \tau_t$.
Evidently, this value is combining the expected reward and the
epistemic uncertainty into a single quantity with $\tau_t$ controlling the
trade-off, and the value is higher for arms with more epistemic
uncertainty.  Now consider the policy $\pi^t(a) \propto \exp(\cevs(a) /
\tau_t)$.  This policy will in general assign greater probability to more
uncertain actions, \ie, the policy is \emph{optimistic}.  We shall
show later that for a carefully selected sequence of temperatures $\tau_t$ we
can ensure that this policy enjoys a $\tilde O(\sqrt{AT})$ Bayes regret bound
for this bandit case. In the more general RL case the posterior over the
Q-values is a complicated function of downstream uncertainties and is not a
simple distribution like a Gaussian, but the intuition is the same.

The choice of the exponential utility may seem arbitrary, but in
fact it is the unique utility function that has the property that the
certainty equivalent value of the sum of two independent random variables is
equal to the sum of their certainty equivalent values \cite{abbas2007invariant,
pfanzag1959general, howard1967value, raiffa68dec}. This property is crucial for
deriving a Bellman recursion, which is necessary for dynamic programming to be
applicable.

\subsection{Optimistic Bellman operator}

A risk-seeking agent would compute the certainty equivalent
value of the Q-values under the endowed utility function and then act to
maximize this value. However, computing the certainty equivalent values in a
full MDP is challenging. The main result (proved in the appendix) is that $\cevs$
satisfies a Bellman \emph{inequality} with a particular optimistic (\ie,
risk-seeking) Bellman operator, which for episode $t$ and timestep $l$ is given
by
\begin{equation}
\label{e-bellmanop}
  \Bc_l^t(\tau, y)(s, a) = \Expect^t \mu_l(s, a)  + \frac{\sigma^2 + (L-l)^2}{2 \tau (n_l^t(s, a) \vee 1)} + \sum_{s^\prime \in \Sc_{l+1}} \Expect^t P_l(s^\prime \mid s, a) (\tau \lse{a^\prime \in \Ac}(y(s^\prime, a^\prime)/\tau))
\end{equation}
for inputs $\tau \geq 0$, $y \in \reals^{|\Sc_l| \times A}$ where
$n_l^t(s,a)$ is the visitation count of the agent to state-action $(s,a)$
at timestep $l$ before episode $t$ and $(\cdot \vee 1) \coloneqq \max(\cdot, 1)$.
Concretely we have that for any $(s,a) \in \Sc\times\Ac$
\[
\cevs_l(s,a) \leq \Bc_l^t(\tau, \cevs_{l+1})(s, a), \quad l=1, \ldots, L.
\]
From this fact we show that the fixed point of the optimistic Bellman
operator yields a guaranteed upper bound on $\cevs$, \ie,
\begin{equation}
\label{e-upperb}
\Big(K^t_l =  \Bc_l^t(\tau, K^t_{l+1}),\ l=1, \ldots, L\Big) \Rightarrow \Big( K^t_l
\geq \cevs_l,\ l=1,\ldots,L \Big).
\end{equation}
We refer to the fixed point as the `K-values' (for knowledge) and we shall show
that they are a sufficiently faithful approximation of $\cevs$ to provide a
Bayes regret guarantee when used instead of the certainty equivalent
values in a policy.

Let us compare the optimistic Bellman operator $\Bc^t$ to the optimal
Bellman operator $\Tc^\star$ defined in \eqref{e-bellman_opt_op}. The first
difference is that the random variables $\mu$ and $P$ are replaced with their
expectation under the posterior; in $\Tc^\star$ they are assumed to be known.
Secondly, the rewards in the optimistic Bellman operator have been augmented
with a bonus that depends on the visitation counts $n^t$. This bonus encourages
the agent to visit state-actions that have been visited less frequently.
Finally, the hard-max of the optimal Bellman operator has been replaced with
a soft-max.  Note that in the limit of zero uncertainty in the MDP (take
$n_l^t(s,a) \rightarrow \infty$ for all $(s,a)$) we have $\Bc_l^t(0, \cdot) =
\Tc_l^\star$ and we recover the optimal Bellman operator, and consequently
in that case $K^t_l(s,a) = \cevs_l(s,a) = Q_l^\star(s,a)$.  In other words, the
optimistic Bellman operator and associated K-values generalize the optimal
Bellman operator and optimal Q-values to the epistemically uncertain case, and
in the limit of zero uncertainty we recover the optimal quantities.

\subsection{Maximum entropy policy}
An agent that acts to maximize its K-values is (approximately)
acting to maximize its risk-seeking utility. In the appendix we show that
the policy that maximizes the expected K-values with \emph{entropy
regularization} is the natural policy to use, which is motivated
by the variational description of the soft-max
\begin{equation}
\label{e-variational}
\max_{\pi_l(s, \cdot) \in \Delta_A} \left(\sum_{a \in \Ac} \pi_l(s, a) K_l^t(s, a) + \tau_t
H(\pi_l(s, \cdot))\right) = \tau_t \lse{a \in \Ac} (K_l^t(s, a)/ \tau_t)
\end{equation}
where $\Delta_A$ is the probability simplex of dimension $A-1$
and $H$ is entropy, \ie, $H(\pi_l(s, \cdot)) = -\sum_{a \in \Ac} \pi_l(s, a) \log\pi_l(s,
a)$ \cite{cover2012elements}. The maximum is achieved by the Boltzmann (or
Gibbs) distribution with temperature $\tau_t$
\begin{equation}
\label{e-policy}
\pi_l^t(s, a) \propto \exp (K_l^t(s, a) / \tau_t).
\end{equation}
This variational principle also arises in statistical mechanics where
Eq.~\eqref{e-variational} refers to the negative Helmholtz free energy and the
distribution in Eq.~\eqref{e-policy} describes the probability that the system
at temperature $\tau_t$ is in a particular `state' \cite{jaynes1957information}.

\subsection{Choosing the risk-seeking parameter / temperature}
The optimistic Bellman operator, K-values, and associated policy depend on the
parameter $\tau_t$. By carefully controlling this parameter we ensure
that the agent balances exploration and exploitation.
We present two ways to do so, the first of which is to follow the
schedule
\begin{equation}
\label{e-betat}
\tau_t = \sqrt{\frac{(\sigma^2 + L^2)SA(1+\log t)}{4 L t \log A}}.
\end{equation}
Alternatively, we find the $\tau_t$ that
yields the tightest bound in the maximal inequality in
\eqref{e-smax-upper}. This turns out to be a convex optimization problem
\begin{align}
\begin{split}
\label{e-cvx-opt-prob}
\mbox{minimize} & \quad \Expect_{s \sim \rho} \left(\tau \lse{a \in \Ac} (K_1(s, a) /
\tau)\right) \\
\mbox{subject to} & \quad K_l \geq \Bc^t_l(\tau, K_{l+1}),\quad l = 1, \ldots, L, \\
& \quad K_{L+1} \equiv 0,
\end{split}
\end{align}
with variables $\tau \geq 0$ and $K \in \reals^{S \times A}$.  This is convex
jointly in $\tau$ and $K$ since the Bellman operator is convex in both
arguments and the perspective of the soft-max term in the
objective is convex \cite{boyd2004convex}. This generalizes the linear
programming formulation of dynamic programming to the case where we have
uncertainty over the parameters that define the MDP \cite{puterman2014markov,
bertsekas2005dynamic}.  Problem \eqref{e-cvx-opt-prob} is an \emph{exponential
cone program} and can be solved efficiently using modern methods \cite{ocpb:16,
scs, ecos, serrano2015algorithms, o2021operator}.

Both of these schemes for choosing $\tau_t$ yield a Bayes regret bound, though
in practice the $\tau_t$ obtained by solving \eqref{e-cvx-opt-prob} tends to
perform better.  Note that since actions are sampled from the stochastic policy
in Eq.~\eqref{e-policy} we refer to K-learning a \emph{randomized}
strategy.  If K-learning is run with the optimal choice of temperature
$\tau^\star_t$ then it is additionally a \emph{stationary} strategy in that the
action distribution depends solely on the posteriors and is otherwise
independent of the time period \cite{russo2018tutorial}.

\subsection{Regret analysis}
\begin{algorithm}[t]
\caption{K-learning for episodic MDPs}
\begin{algorithmic}
\setstretch{1.1}
\State {\bf Input:} MDP $\mathcal{M} = \{\Sc, \Ac, R, P, L, \rho\}$,
\For{episode $t=1,2,\ldots$}
\State calculate $\tau_t$ using \eqref{e-betat} or \eqref{e-cvx-opt-prob}
\State set $K^t_{L+1} \equiv 0$
\State compute $K^t_l = \Bc^t_l(\tau_t, K^t_{l+1})$, for $l=L,\ldots,1$, using \eqref{e-bellmanop}
\State execute policy $\pi_l^t(s, a) \propto \exp(K_l^t(s, a)/ \tau_t)$, for $l=1,\ldots,L$
\EndFor
\end{algorithmic}
\label{a-klearning}
\end{algorithm}

\newcommand{\brthmtext}{
Under assumption \ref{ass1} the K-learning algorithm~\ref{a-klearning} satisfies
Bayes regret bound
\begin{align}
\begin{split}
\bayesregret(T) &\leq 2 \sqrt{(\sigma^2 + L^2) S A T \log A (1 + \log T/ L)}\\
&= \tilde O(L \sqrt{S A T}).
\end{split}
\end{align}
}
\newtheorem*{T1}{Theorem~\ref{t-kl_mdp}}
\begin{theorem}
\label{t-kl_mdp}
\brthmtext
\end{theorem}
The full proof is included in the appendix. The main challenge is showing that
the certainty equivalent values $\cevs$ satisfy the Bellman inequality with the
optimistic Bellman operator \eqref{e-bellmanop}. This is used to show that the
K-values upper bound $\cevs$ (Eq.~\eqref{e-upperb}) from which we derive
the following maximal inequality
\begin{align}
\begin{split}
\label{e-smax-upper}
\Expect^t \max_a Q^\star_l(s,a) \leq \tau_t \lse{a \in \Ac} (\cevs_l(s,a) / \tau_t) \leq
\tau_t \lse{a \in \Ac} (K^t_l(s,a) /\tau_t).
\end{split}
\end{align}
From this and the Bellman recursions that the K-values and the Q-values must
satisfy, we can `unroll' the Bayes regret \eqref{e-bayesregret1} over the MDP.
Using the variational description of the soft-max in Eq.~\eqref{e-variational}
we can cancel out the expected reward terms leaving us with a sum over
`uncertainty' terms. Since the uncertainty is reduced by the
agent visiting uncertain states we can bound the remaining terms using a standard
pigeonhole argument. Finally, the temperature $\tau_t$ is a free-parameter for
each episode $t$, so we can choose it so as to minimize the upper bound. This
yields the final result.

The Bayes regret bound in the above theorem matches the best known bound for
Thompson sampling up to log factors \cite{osband2013more}.  Moreover, the above
regret bound is within a factor of $\sqrt{L}$ of the known information
theoretic lower bound \cite[Appendix D]{jin2018q}.

Intuitively speaking, K-values are higher where the agent has high epistemic
uncertainty. Higher K-values will make the agent more likely to take the actions
that lead to those states.  Over time states with high uncertainty will be
visited and the uncertainty about them will be resolved. The temperature
parameter $\tau_t$ is controlling the balance between exploration
and exploitation.

\section{Numerical experiments}
In this section we compare the performance of both the temperature scheduled and
optimized temperature variants of K-learning against several other methods in
the literature. %
We consider a small tabular MDP called \emph{DeepSea} \cite{osband2017deep}
shown in Figure~\ref{fig:deepsea-cartoon}, which can be thought of as an unrolled
version of the RiverSwim environment \cite{strehl2008analysis}.  This MDP can be
visualized as an $L \times L$ grid where the agent starts at the top row and
leftmost column. At each time-period the agent can move left or right and
descends one row. The only positive reward is at the bottom right cell.  In
order to reach this cell the agent must take the `right' action every timestep.
After choosing action `left' the agent receives a random reward with zero mean, and after
choosing right the agent receives a random reward with small negative mean. At the
bottom rightmost corner of the grid the agent receives a random reward with mean one.
Although this is a toy example it provides a challenging `needle in a
haystack' exploration problem. Any algorithm that does exploration via a simple
heuristic like local dithering will take time exponential in the depth $L$ to
reach the goal.
Policies that perform deep exploration can learn much faster
\cite{osband2013more, o2018uncertainty}.

\begin{figure}
  \centering
  \includegraphics[width=0.7\linewidth]{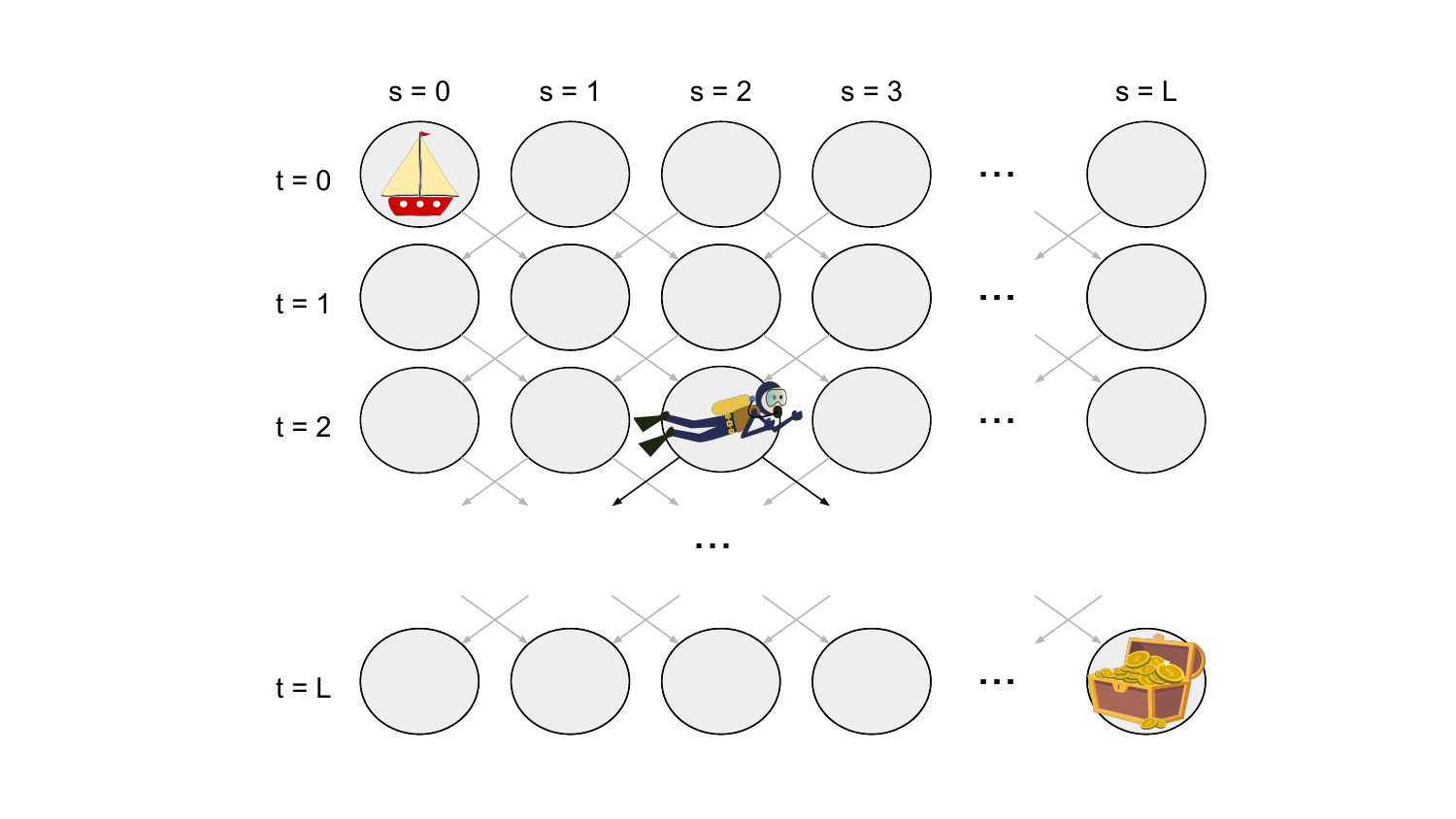}
  \vspace{5mm}
  \captionof{figure}{The DeepSea MDP}
  \label{fig:deepsea-cartoon}
\end{figure}%

In Figure~\ref{fig:deepsea-regret} we show the time required to `solve' the
problem as a function of the depth of the environment, averaged over $5$ seeds
for each algorithm. We define `time to solve' to be the first episode at which
the agent has reached the rewarding state in at least $10\%$ of the episodes so
far. If an agent fails to solve an instance within $10^5$ episodes we do not
plot that point, which is why some of the traces appear to abruptly stop. We
compare two dithering approaches, Q-learning with epsilon-greedy ($\epsilon =
0.1$) and soft-Q-learning \cite{haarnoja17} ($\tau = 0.05$), against principled
exploration strategies RLSVI \cite{osband2017deep}, UCBVI
\cite{azar2017minimax}, optimistic Q-learning (OQL) \cite{jin2018q}, BEB
\cite{kolter2009near}, Thompson sampling \cite{osband2013more} and two variants
of K-learning, one using the $\tau_t$ schedule \eqref{e-betat} and the other
using the optimal choice $\tau_t^\star$ from solving \eqref{e-cvx-opt-prob}.
Soft Q-learning is similar to K-learning with two major differences: the
temperature term is a fixed hyperparameter and there is no optimism bonus added
to the rewards. These differences prevent soft Q-learning from satisfying a
regret bound and typically it cannot solve difficult exploration tasks in
general \cite{o2020making}. We also ran comparisons against BOLT
\cite{araya2012near}, UCFH \cite{dann2015sample}, and UCRL2
\cite{jaksch2010near} but they did not perform much better than the dithering
approaches and contributed to the clutter in the figure so we do not show them.

As expected, the two `dithering' approaches are unable to handle the problem as
the depth exceeds a small value; they fail to solve the problem within $10^5$
episodes for problems larger than $L=6$ for epsilon-greedy and $L=14$ for
soft-Q-learning. These approaches are taking time exponential in the size of the
problem to solve the problem, which is seen by comparing their performance to
the grey dashed line which plots $2^{L-1}$. The other approaches scale more
gracefully, however clearly Thompson sampling and K-learning are the most
efficient. The optimal choice of K-learning appears to perform slightly better
than the scheduled temperature variant, which is unsurprising since it is
derived from a tighter upper bound on the regret.

In Figures \ref{fig-deepsea-kvals} and \ref{fig-deepsea-softqvals} we show the
progress of K-learning using $\tau_t^\star$ and Soft Q-learning for a single
seed running on the $L=50$ depth DeepSea environment.   In the top row of each
figure we show the value of each state over time, defined for K-learning as
\begin{equation}
\label{e-soft-value}
\tilde V_l^t(s) = \tau_t \lse{a \in \Ac} (K_l^t(s,a) / \tau_t),
\end{equation}
and analogously for Soft Q-learning.  The
bottom row shows the log of the visitation counts over time. Although both
approaches start similarly, they quickly diverge in their behavior. If we examine
the K-learning plots, it is clear that the agent is visiting more of the bottom
row as the episodes proceed. This is driven by the fact that the value,
which incorporates the epistemic uncertainty, is high for the unvisited states.
Concretely, take the $t=300$ case; at this point the K-learning agent has not
yet visited the rewarding bottom right state, but the value is very high for
that region of the state space and shortly after this it reaches the reward for
the first time. By the $t=1000$ plot the agent is travelling along the diagonal
to gather the reward consistently.
Contrast this to Soft Q-learning, where the agent does not make it even halfway
across the grid after $t=1000$ episodes. This is because the soft Q-values
do not capture any uncertainty about the environment, so the agent has no
incentive to explore and visit new states. The only exploration that soft
Q-learning is performing is merely the local dithering arising from using a
Boltzmann policy with a nonzero temperature. Indeed, the soft value function
barely changes in this case since the agent is consistently gathering zero-mean
rewards; any fluctuation in the value function arises merely from the
noise in the random rewards.

\begin{figure}
  \centering
  \includegraphics[width=1\linewidth]{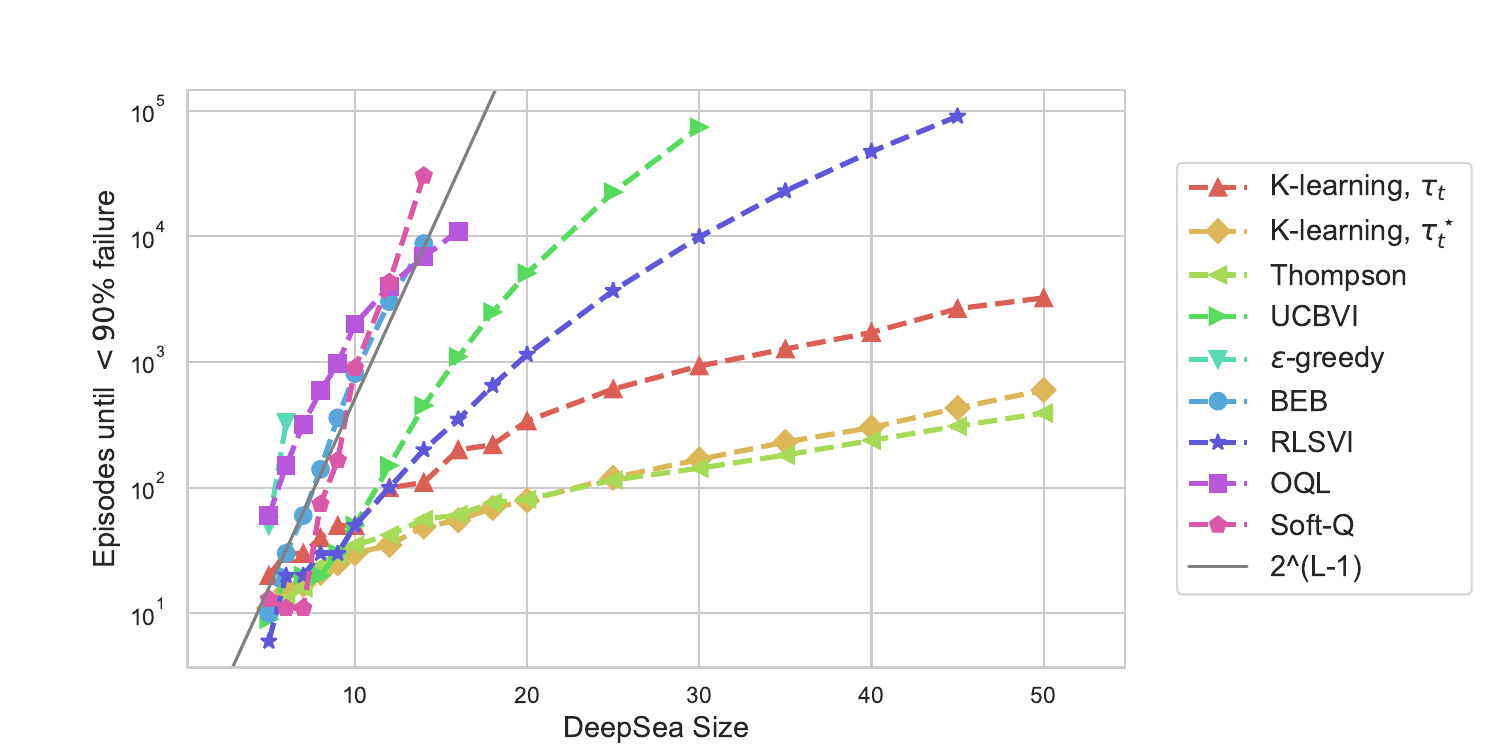}
  \captionof{figure}{Learning time on DeepSea.}
  \label{fig:deepsea-regret}
\end{figure}%

\begin{figure}[h]
\centering
\begin{subfigure}[ht]{0.194\textwidth}
\includegraphics[width=\linewidth]{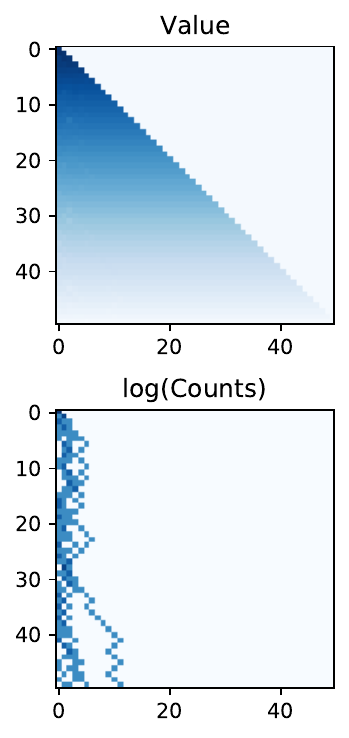}
\caption{$t=5$.}
\end{subfigure}
\begin{subfigure}[ht]{0.194\textwidth}
\includegraphics[width=\linewidth]{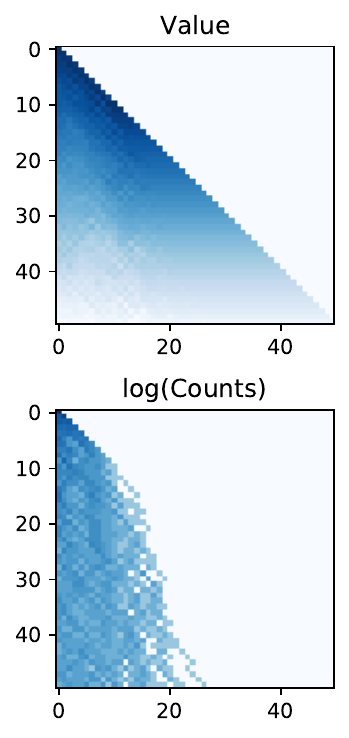}
\caption{$t=50$.}
\end{subfigure}
\begin{subfigure}[ht]{0.194\textwidth}
\includegraphics[width=\linewidth]{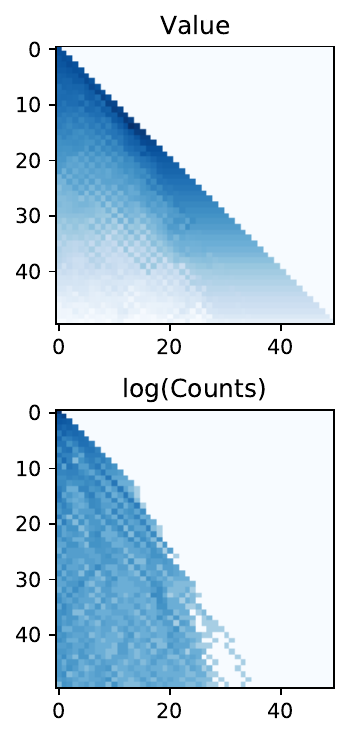}
\caption{$t=100$.}
\end{subfigure}
\begin{subfigure}[ht]{0.194\textwidth}
\includegraphics[width=\linewidth]{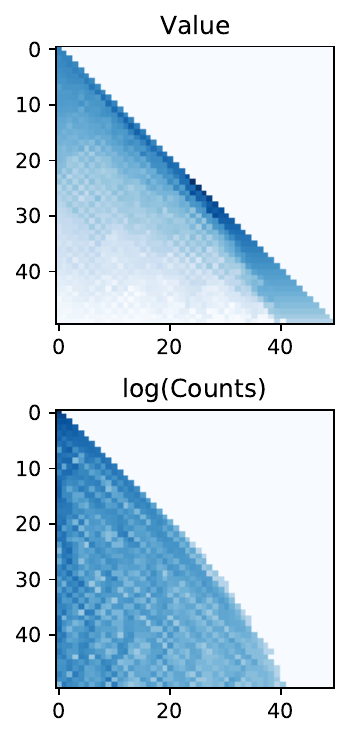}
\caption{$t=300$.}
\end{subfigure}
\begin{subfigure}[ht]{0.194\textwidth}
\includegraphics[width=\linewidth]{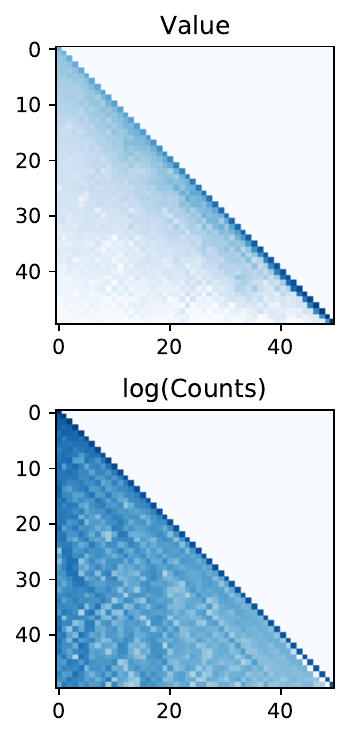}
\caption{$t=1000$.}
\end{subfigure}
\caption{Value (as defined in \eqref{e-soft-value}) and log visitation count for
each state in DeepSea under K-learning using $\tau_t^\star$.
Darker color indicates larger value.}
\label{fig-deepsea-kvals}
\end{figure}

\begin{figure}[h]
\centering
\begin{subfigure}[ht]{0.194\textwidth}
\includegraphics[width=\linewidth]{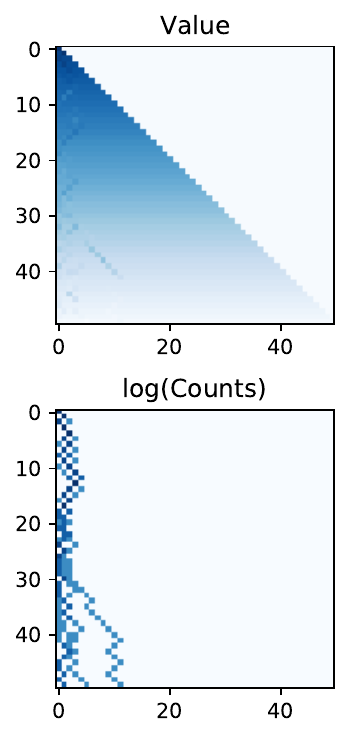}
\caption{$t=5$.}
\end{subfigure}
\begin{subfigure}[ht]{0.194\textwidth}
\includegraphics[width=\linewidth]{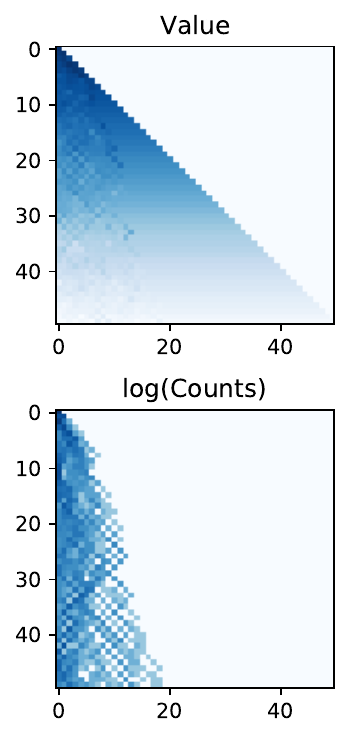}
\caption{$t=50$.}
\end{subfigure}
\begin{subfigure}[ht]{0.194\textwidth}
\includegraphics[width=\linewidth]{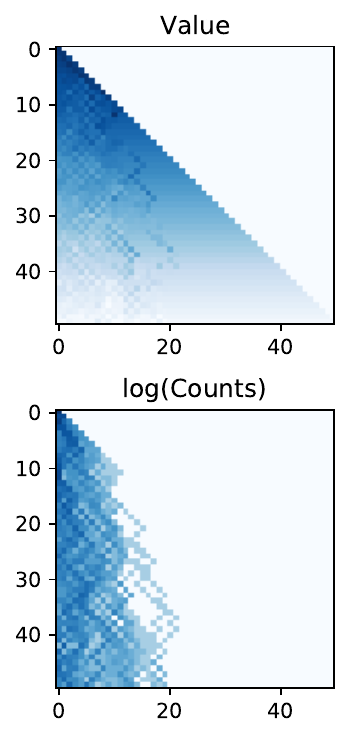}
\caption{$t=100$.}
\end{subfigure}
\begin{subfigure}[ht]{0.194\textwidth}
\includegraphics[width=\linewidth]{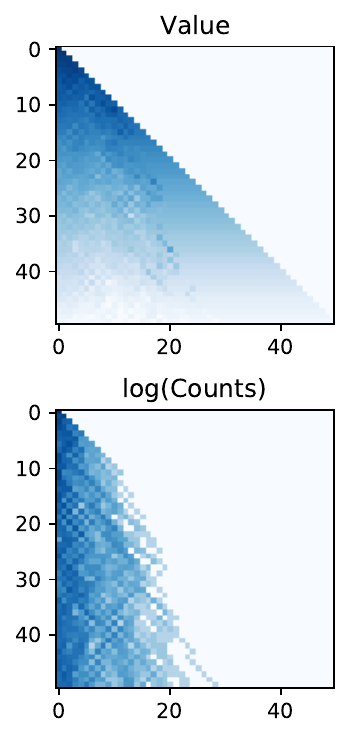}
\caption{$t=300$.}
\end{subfigure}
\begin{subfigure}[ht]{0.194\textwidth}
\includegraphics[width=\linewidth]{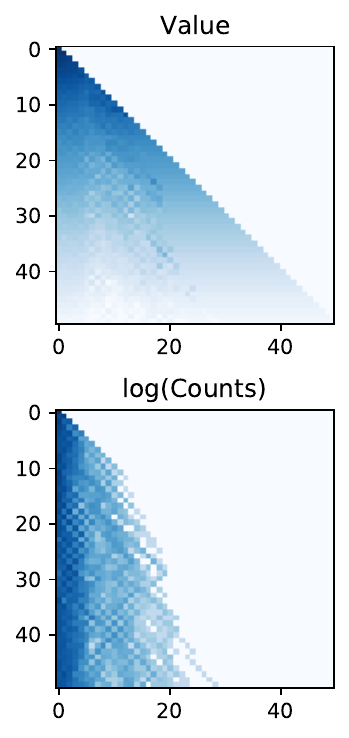}
\caption{$t=1000$.}
\end{subfigure}
\caption{Value  (as defined in \eqref{e-soft-value}) and log visitation count
for each state in DeepSea under soft Q-learning. Darker color indicates larger
value.}
\label{fig-deepsea-softqvals}
\end{figure}

\section{Conclusions}
In this work we endowed a reinforcement learning agent with a
\emph{risk-seeking} utility, which encourages the agent to take actions that
lead to less epistemically certain states. This yields a Bayesian algorithm with
a bound on regret which matches the best-known regret bound for Thompson
sampling up to log factors and is close to the known lower bound.  We call the
algorithm `K-learning', since the `K-values' capture something about the
epistemic \emph{knowledge} that the agent can obtain by visiting each
state-action. In the limit of zero uncertainty the K-values reduce to the
optimal Q-values.

Although K-learning and Thompson sampling have similar theoretical and empirical
performance, K-learning has some advantages. For one, it was recently shown that
K-learning extends naturally to the case of two-player games and continues to
enjoy a sub-linear regret bound, whereas Thompson sampling can suffer linear
regret \cite{klearninggames}. Secondly, Thompson sampling requires sampling from
the posterior over MDPs and solving the sampled MDP exactly at each episode.
This means that Thompson sampling does not have a `fixed' policy at any given
episode since it is determined by the posterior information plus a sampling
procedure. This is typically a deterministic policy, and it can vary
significantly from episode to episode.  By contrast K-learning has a single
fixed policy at each episode, the Boltzmann distribution over the K-values, and
it is determined entirely from the posterior information (\ie, no sampling).
Moreover, K-learning requires solving a Bellman equation that changes slowly as
more data is accumulated, so the optimal K-values at one episode are close to
the optimal K-values for the next episode, and similarly for the policies. This
suggests the possibility of approximating K-learning in an online manner and
making use of modern deep RL techniques such as representing the K-values using
a deep neural network \cite{mnih-dqn-2015}.  This is in line with the purpose of
this paper, which is not to get the best theoretical regret bound, but instead
to derive an algorithm that is close to something that is practical to implement
in a real RL setting.  Soft-max updates, maximum-entropy RL, and related
algorithms are very popular in deep RL. However, they are not typically well
motivated and they cannot perform deep-exploration. This paper tackles both
those problems since the soft-max update, entropy regularization, and
deep-exploration all fall out naturally from a utility maximization point of
view. The popularity and success of these other approaches, despite their
evident shortcomings in data efficiency, suggest that incorporating changes
derived from K-learning could yield big performance improvements in real-world
settings. We leave exploring
that avenue to future work.

\begin{ack}
I would like to thank Ian Osband, Remi Munos, Vlad Mnih, Pedro Ortega, Sebastien
Bubeck, Csaba Szepesv\'ari, and Yee Whye Teh for valuable discussions and clear
insights. I received no specific funding for this work.
\end{ack}

\bibliographystyle{abbrv}
\bibliography{klearning}
\newpage

\section*{Checklist}

\begin{enumerate}

\item For all authors...
\begin{enumerate}
  \item Do the main claims made in the abstract and introduction accurately reflect the paper's contributions and scope?
    \answerYes{}
  \item Did you describe the limitations of your work?
    \answerYes{}
  \item Did you discuss any potential negative societal impacts of your work?
    \answerNo{This paper is a purely theoretical work with little
societal impact.}
  \item Have you read the ethics review guidelines and ensured that your paper conforms to them?
    \answerYes{}
\end{enumerate}

\item If you are including theoretical results...
\begin{enumerate}
  \item Did you state the full set of assumptions of all theoretical results?
    \answerYes{}
	\item Did you include complete proofs of all theoretical results?
    \answerYes{}
\end{enumerate}

\item If you ran experiments...
\begin{enumerate}
  \item Did you include the code, data, and instructions needed to reproduce the main experimental results (either in the supplemental material or as a URL)?
    \answerNo{We included a description of the data generation process for the simulations we ran.}
  \item Did you specify all the training details (e.g., data splits, hyperparameters, how they were chosen)?
    \answerNA{These experiments involved no training on external data.}
	\item Did you report error bars (e.g., with respect to the random seed after running experiments multiple times)?
    \answerNo{I produced Fig 2 with error bars, but it looked very cluttered and obscured the message without providing any additional insight.}
	\item Did you include the total amount of compute and the type of resources used (e.g., type of GPUs, internal cluster, or cloud provider)?
    \answerYes{Included in appendix.}
\end{enumerate}

\item If you are using existing assets (e.g., code, data, models) or curating/releasing new assets...
\begin{enumerate}
  \item If your work uses existing assets, did you cite the creators?
    \answerNA {}
  \item Did you mention the license of the assets?
    \answerNA{}
  \item Did you include any new assets either in the supplemental material or as a URL?
    \answerNA{}
  \item Did you discuss whether and how consent was obtained from people whose data you're using/curating?
    \answerNA{}
  \item Did you discuss whether the data you are using/curating contains personally identifiable information or offensive content?
    \answerNA{}
\end{enumerate}

\item If you used crowdsourcing or conducted research with human subjects...
\begin{enumerate}
  \item Did you include the full text of instructions given to participants and screenshots, if applicable?
    \answerNA{}
  \item Did you describe any potential participant risks, with links to Institutional Review Board (IRB) approvals, if applicable?
    \answerNA{}
  \item Did you include the estimated hourly wage paid to participants and the total amount spent on participant compensation?
    \answerNA{}
\end{enumerate}

\end{enumerate}
\newpage

\appendix

\section{Appendix}
This appendix is dedicated to proving Theorem \ref{t-kl_mdp}.  First, we
introduce some notation.  The cumulant generating function of random variable
$X:\Omega \rightarrow \reals$ is given by
\[
G^{X}(\beta) = \log \Expect \exp(\beta X).
\]
We shall denote the cumulant generating function of $\mu_l(s,a)$ at time $t$
as $G^{\mu|t}_l(s,a,\cdot)$ and similarly the cumulant generating function
of $Q_l^\star(s,a)$ at time $t$ as $G_l^{Q|t}(s,a, \cdot )$, specifically
\[
G^{\mu|t}_l(s,a,\beta) = \log \Expect^t \exp(\beta \mu_l(s,a)), \quad
G^{Q|t}_l(s,a,\beta) = \log \Expect^t \exp(\beta Q^\star_l(s,a)),
\]
for $l=1,\ldots,L$.

\begin{T1}
\brthmtext
\end{T1}

\begin{proof}
Lemma~\ref{l-bellman-ineq} tells us that $G^{Q|t}(s,a,\beta)$ satisfies a
Bellman inequality for any $\beta \geq 0$. This implies that for fixed $\tau_t \geq 0$
the certainty equivalent values
$\cevs_l = \tau_t G_l^{Q|t}(s,a,1/\tau_t)$ satisfy a Bellman inequality with
optimistic Bellman operator $\Bc^t$ defined in Eq.~\eqref{e-bellmanop},
\ie,
\[
  \cevs_l \leq \Bc_l^t(\tau_t, \cevs_{l+1}),
\]
for $l=1,\ldots,L$, where $\cevs_{L+1} \equiv 0$.
By construction the K-values $K^t$ are the unique fixed point of the
optimistic Bellman operator. That is, $K^t$ has $K^t_{L+1} \equiv 0$ and
\begin{equation}
\label{e-kbellmaneq}
K^t_l = \Bc^t_l(\tau_t, K^t_{l+1})
\end{equation}
for $l=1, \ldots, L$. Since log-sum-exp is nondecreasing it
implies that the operator $\Bc^t_l(\tau, \cdot)$ is
nondecreasing for any $\tau\geq 0$, \ie, if $x \geq y$ pointwise then
$\Bc^t_l(\tau, x) \geq \Bc^t_l(\tau,  y)$ pointwise for each $l$. Now assume
that for some $l$ we have $K^t_{l+1} \geq \cevs_{l+1}$,
then
\[
  K^t_l  = \Bc^t_l(\tau_t,  K^t_{l+1})  \geq \Bc^t_l(\tau_t,
  \cevs_{l+1})  \geq \cevs_l,
\]
and the base case holds since $K^t_{L+1}=
\cevs_{L+1}\equiv 0$.
This fact, combined with Lemma \ref{l-max-ineq} implies that
\begin{equation}
\label{e-max_k}
\Expect^t \max_a Q_l^{\star}(s, a) \leq \tau_t \lse{a \in \Ac} (\cevs_l(s,
a)/\tau_t) \leq \tau_t  \lse{a \in \Ac}(K_l^t(s, a) / \tau_t)
\end{equation}
since log-sum-exp is increasing and $\tau_t \geq 0$.
The following variational identity yields the policy that the agent will follow:
\[
\tau_t  \lse{a \in \Ac}(K_l^t(s, a)/ \tau_t) = \max_{\pi_l(s, \cdot) \in \Delta_A}
\left(\sum_{a \in \Ac} \pi_l(s, a) K_l^t(s, a) + \tau_t H(\pi_l(s, \cdot))\right)
\]
for any state $s$, where $\Delta_A$ is the probability simplex of dimension
$A-1$ and $H$ denotes the
entropy, \ie, $H(\pi(s, \cdot)) = -\sum_{a \in \Ac} \pi(s, a) \log\pi(s,a)$
\cite{cover2012elements}. The maximum is achieved by the policy
\[
\pi^t_l(s, a) \propto \exp (K_l^t(s, a) / \tau_t).
\]
This comes from taking the Legendre transform of negative entropy term
(equivalently, log-sum-exp and negative entropy are convex conjugates
\cite[Example 3.25]{boyd2004convex}).  The fact that \eqref{e-policy} achieves
the maximum is readily verified by substitution.

Now we consider the Bayes regret of an agent following policy
\eqref{e-policy}, starting from \eqref{e-bayesregret1} we have
\begin{align}
\begin{split}
\label{e-reg-ub}
\bayesregret(T)
  &\overset{\scriptscriptstyle (a)}{=}\Expect \sum_{t=1}^{\episodes}
\Expect_{s \sim \rho} \Expect^t(V_1^\star(s) - V_1^{\pi^t}(s))\\
&=\Expect \sum_{t=1}^{\episodes}
\Expect_{s \sim \rho} \Expect^t\Big(\max_{a} Q_1^\star(s, a) - \sum_{a \in \Ac}
  \pi_1^t(s, a) Q_1^{\pi^t}(s, a)\Big)\\
  &\overset{\scriptscriptstyle (b)}{\leq}\Expect \sum_{t=1}^{\episodes}
\Expect_{s \sim \rho} \Big(
\tau_t \lse{a \in \Ac} ( K_1^t(s, a) / \tau_t)
  - \sum_{a \in \Ac} \pi_1^t(s, a) \Expect^tQ_1^{\pi^t}(s, a)\Big) \\
&\overset{\scriptscriptstyle (c)}{\leq} \Expect \sum_{t=1}^{\episodes}
  \Expect_{s \sim \rho} \bigg(
 \sum_{a \in \Ac} \pi_1^t(s, a)\Big( K_1^t(s, a) -
  \Expect^t Q_1^{\pi^t}(s, a)\Big) + \tau_tH(\pi_1^t(s))\bigg)
\end{split}
\end{align}
where (a) follows from the tower property of conditional expectation where the
outer expectation is with respect to $\Fc_1, \Fc_2, \ldots$, (b) is due to
\eqref{e-max_k} and the fact that $\pi^t$ is $\Fc_t$-measurable, and (c) is due
to the fact that the policy the agent is following is the policy
\eqref{e-policy}.  If we denote by
\[
  \Delta_l^t(s) = \sum_{a \in \Ac} \pi_l^t(s, a)\left(K_l^t(s, a) - \Expect^t
Q_l^{\pi^t}(s, a)\right) + \tau_tH(\pi_l^t(s))
\]
then we can write the previous bound simply as
\[
\bayesregret(T)\leq \Expect
\sum_{t=1}^{\episodes}\Expect_{s \sim \rho}\Delta_1^t(s).
\]
We can interpret $\Delta^t_l(s)$ as a bound on the expected regret in that
episode when started at state $s$. Let us denote
\[
\tilde G_l^{\mu | t}(s, a, \beta) =
  G_l^{\mu | t}(s, a, \beta)+ \frac{(L - l)^2 \beta^2}{2 (n_l^t(s, a) \vee 1)}.
\]
Now we shall show that for a fixed $\pi^t$ and $\tau_t \geq 0$ the quantity $\Delta^t$ satisfies the
following Bellman recursion:
\begin{equation}
\label{e-big-delta}
\Delta_l^t(s) = \tau_t H(\pi_l^t(s)) + \sum_{a \in \Ac} \pi_l^t(s, a)\left(\delta_l^t(s, a,
  \tau_t) + \sum_{s^\prime \in \Sc_{l+1}} \Expect^t (P_l(s^\prime \mid s, a)) \Delta_{l+1}^t(s^\prime)  \right)
\end{equation}
for $s \in \Sc$, $l=1, \ldots, L$, and $\Delta^t_{L+1} \equiv 0$, where
\begin{equation}
\label{e-little-delta}
  \delta_l^{t}(s, a, \tau) = \tau \tilde G_l^{\mu | t}(s, a, 1/ \tau) - \Expect^t \mu_l(s, a) \leq \frac{\sigma^2 +
(L - l)^2}{2 \tau (n_l^t(s, a) \vee 1)},
\end{equation}
where the inequality follows from assumption \ref{ass1} which allows us to
bound $G_l^{\mu | t}$ as
\[
  \tau G_l^{\mu | t}(s, a, 1/\tau) \leq
  \Expect^t \mu_l(s, a)  + \frac{\sigma^2}{2 \tau (n_l^t(s, a) \vee 1)}
\]
for all $\tau \geq 0$.
We have that
\begin{align}
\begin{split}
\label{e-app-qvals}
\Expect^t Q_l^{\pi^t}(s, a)
  &\overset{\scriptscriptstyle (a)}{=} \Expect^t\Big(\mu_l(s, a) + \sum_{s^\prime \in \Sc_{l+1}}
P_l(s^\prime \mid s, a) V_{l+1}^{\pi^t}(s^\prime)\Big)\\
  &\overset{\scriptscriptstyle (b)}{=} \Expect^t \mu_l(s, a) + \sum_{s^\prime \in \Sc_{l+1}}
\Expect^t P_l(s^\prime \mid s, a) \Expect^t V_{l+1}^{\pi^t}(s^\prime)\\
  &\overset{\scriptscriptstyle (c)}{=} \Expect^t \mu_l(s, a) + \sum_{s^\prime \in \Sc_{l+1}}
\Expect^t P_l(s^\prime \mid s, a)  \sum_{a^\prime \in \Ac} \pi_{l+1}^t(s^\prime, a^\prime)
\Expect^t Q_{l+1}^{\pi^t}(s^\prime, a^\prime),
\end{split}
\end{align}
where (a) is the Bellman Eq.~\eqref{e-bellman_q}, (b) holds due to
the fact that the transition function and the value function at the next state are
conditionally independent, (c) holds since
$\pi^t$ is $\Fc_t$ measurable.

Now we expand the definition of $\Delta^t$, using the Bellman equation that
the K-values satisfy and Eq.~\eqref{e-app-qvals} for the Q-values
\begin{align*}
  \Delta_l^t(s) &= \tau_t H(\pi_l^t(s)) +
  \sum_{a \in \Ac} \pi_l^t(s, a)\Big(\tau_t \tilde G_l^{\mu | t}(s, a, 1/\tau_t) - \Expect^t \mu_l(s, a) + \\
  &\qquad\sum_{s^\prime \in \Sc_{l+1}} \Expect^t P_l(s^\prime \mid s, a) \big( \tau_t \lse{a^\prime \in \Ac}
    K_{l+1}^t(s^\prime, a^\prime) / \tau_t - \sum_{a^\prime \in \Ac} \pi_{l+1}^t(s^\prime,
a^\prime) \Expect^t Q_{l+1}^{\pi^t}(s^\prime, a^\prime)\big)\Big)\\
  &= \tau_t H(\pi_l^t(s)) +
  \sum_{a \in \Ac} \pi_l^t(s, a)\bigg(\delta_l^t(s, a, \tau_t) +\\
  &\qquad\sum_{s^\prime \in \Sc_{l+1}} \Expect^t P_l(s^\prime \mid s, a)\Big(\tau_t
H(\pi_{l+1}^t(s^\prime)) + \sum_{a^\prime \in \Ac} \pi_{l+1}^t(s^\prime,
a^\prime)\Big(K_{l+1}^t(s^\prime, a^\prime) - \Expect^t Q_{l+1}^{\pi^t}(s^\prime, a^\prime)\Big) \Big) \bigg)\\
 &=\tau_t H(\pi_l^t(s)) +
  \sum_{a \in \Ac} \pi_l^t(s, a)\bigg(\delta_l^t(s, a, \tau_t) +
  \sum_{s^\prime \in \Sc_{l+1}} \Expect^t P_l(s^\prime \mid s, a)\Delta_{l+1}^t(s^\prime) \bigg),
\end{align*}
where we used the variational representation \eqref{e-variational}.
We shall use
this to `unroll' $\Delta^t$ along the MDP, allowing us to write the regret
upper bound using only \emph{local} quantities.

An \emph{occupancy measure} is the probability that the agent finds
itself in state $s$ and takes action $a$.
Let $\lambda_l^t(s, a)$ be the expected occupancy measure for
state $s$ and action $a$ under the policy $\pi^t$ at time $t$, that is
$\lambda_1^t(s, a) =
\pi_1^t(s, a)\rho(s)$, and then it satisfies the forward recursion
\[
  \lambda_{l+1}^t(s^\prime, a^\prime) = \pi_l^t(s^\prime, a^\prime)\sum_{(s, a) \in \Sc_l \times \Ac}
\Expect^t (P_l(s^\prime \mid s, a)) \lambda_l^t(s, a),
\]
for $l=1, \ldots, L$, and note that $\sum_{(s, a) \in \Sc_l \times \Ac}
\lambda_l^t(s, a) = 1$ and so it is a valid probability
distribution over $\Sc_l \times \Ac$ for each $l$. Now let us define the following function
\begin{equation}
\label{e-epregret-def}
 \epregret^t(\tau, \lambda) = \sum_{l=1}^L \sum_{(s, a)\in \Sc_l \times \Ac}\lambda^t_l(s, a)\left(
\tau H\left(\frac{\lambda_l^t(s)}{\sum_b \lambda_l^t(s,b)}\right) +  \delta_l^t(s, a, \tau) \right).
\end{equation}
where $\lambda^t(s)$ is the vector corresponding to the occupancy measure
values at state $s$.
One can see that by unrolling the definition of $\Delta^t$ in \eqref{e-big-delta}
we have that
\begin{align*}
\Expect_{s \sim \rho} \Delta_1^t(s) = \epregret^t(\tau_t, \lambda^t).
\end{align*}
In order to prove the Bayes regret bound, we must bound this $\epregret^t$
function. For the case of $\tau_t$ annealed according to the schedule
of \eqref{e-betat} and the associated expected occupancy measure $\lambda^t$ we do this using lemma \ref{l-delta_unroll}.
For the case of $\tau^\star_t$ the solution to \eqref{e-cvx-opt-prob} and the
associated expected occupancy measure $\lambda^{t\star}$
lemma \ref{l-cvx-opt-regret} proves that
\[
\epregret^t(\tau^\star_t, \lambda^{t\star}) \leq \epregret^t(\tau_t, \lambda^t),
\]
and so it satisfies the same regret bound as the
annealed parameter.  This result concludes the proof.
\end{proof}

\subsection{Proof of Bellman inequality lemma~\ref{l-bellman-ineq}}
\begin{lemma}
\label{l-bellman-ineq}
The cumulant generating function of the posterior for the optimal Q-values satisfies
the following Bellman inequality for all $\beta \geq 0$, $l=1,\ldots,L$:
\[
G_l^{Q | t}(s, a, \beta) \leq \tilde G_l^{\mu | t}(s, a, \beta) +
  \sum_{s^\prime \in \Sc_{l+1}} \Expect^t P_l(s^\prime \mid s, a) \lse{a^\prime \in \Ac}
  G_{l+1}^{Q | t}(s^\prime, a^\prime, \beta).
\]
where
\[
\tilde G_l^{\mu | t}(s, a, \beta) =
  G_l^{\mu | t}(s, a, \beta)+ \frac{(L - l)^2 \beta^2}{2 (n_l^t(s, a) \vee 1)}.
\]
\end{lemma}
\begin{proof}
We begin by applying the definition of the cumulant generating function
\begin{align}
\begin{split}
\label{e-app-partial}
  G_l^{Q | t}(s, a, \beta) &= \log \Expect^t \exp \beta Q_l^\star(s, a)\\
   &= \log \Expect^t \exp \Big(\beta \mu_l(s, a) + \beta \sum_{s^\prime \in \Sc_{l+1}}P_l(s^\prime \mid s, a)
    V_{l+1}^\star(s^\prime)\Big)\\
   &= G_l^{\mu | t}(s, a, \beta) + \log \Expect^t \exp \Big(\beta \sum_{s^\prime \in \Sc_{l+1}}P_l(s^\prime \mid s, a)
    V_{l+1}^\star(s^\prime)\Big)
\end{split}
\end{align}
where $G_l^{\mu|t}$ is the cumulant generating function for $\mu$,
and where the first equality is the Bellman equation for $Q^\star$, and the
second one follows the fact that
$\mu_l(s, a)$ is conditionally independent of downstream quantities.
Now we must deal with the second term in the above expression.

Assumption~\ref{ass1} says that the prior over the transition function
$P_l(~\cdot \mid s, a)$ is Dirichlet, so let us denote the parameter of the
Dirichlet distribution $\alpha_l^0(s, a) \in \reals_+^{S}$ for each
  $(s, a)$, and we make the additional mild assumption that $\sum_{s^\prime \in \Sc_{l+1}}
\alpha_l^0(s, a, s^\prime) \geq 1$, \ie, we start with a total pseudo-count of
at least one for every state-action. Since the likelihood for the transition
function is a Categorical distribution, conjugacy of the categorical and
Dirichlet distributions implies that the posterior over $P_l(~\cdot \mid s, a)$ at
time $t$ is Dirichlet with parameter $\alpha_l^t(s, a)$, where
\[
  \alpha_l^t(s, a, s^\prime) = \alpha_l^0(s, a, s^\prime) + n^t_l(s,
  a, s^\prime)
\]
for each $s^\prime \in \Sc_{l+1}$, where $n_l^t(s, a, s^\prime) \in \nats$
is the number of times the agent has been in state
$s$, taken action $a$, and transitioned to state $s^\prime$ at timestep $l$, and note that
$\sum_{s^\prime \in \Sc_{l+1}} n_l^t(s, a, s^\prime) = n_l^t(s, a)$, the
total visit count to $(s, a)$.

Our analysis will make use of the following definition and associated lemma from
\cite{osband2017gaussian}.  Let $X$ and $Y$ be random variables, we say that $X$
is stochastically optimistic for $Y$, written $X \geq_{SO} Y$, if $\Expect u(X)
\geq \Expect u(Y)$ for any convex increasing function $u$.  Stochastic optimism
is closely related to the more familiar concept of second-order stochastic
dominance, in that $X$ is stochastically optimistic for $Y$ if and only if $-Y$
second-order stochastically dominates $-X$ \cite{hadar1969rules}.  We use this
definition in the next lemma.
\begin{lemma}
\label{l-ian}
Let $Y$ = $\sum_{i=1}^n A_i b_i$ for fixed $b \in \reals^n$ and random
variable $A$, where $A$ is Dirichlet with parameter $\alpha \in \reals^n$, and let $X \sim
\mathcal{N}(\mu_X, \sigma_X^2)$ with $\mu_X \geq \frac{\sum_i \alpha_i
b_i}{\sum_i \alpha_i}$ and $\sigma_X^2 \geq (\sum_i
\alpha_i)^{-1}\mathrm{Span}(b)^2$, where $\mathrm{Span(b)} = \max_i b_i - \min_j
b_j$, then $X \geq_{SO} Y$.
\end{lemma}
For the proof see \cite{osband2017gaussian}.
In our case, in the notation of the lemma~\ref{l-ian}, $A$ will represent the
transition function probabilities, and $b$ will represent the optimal values of
the next state, \ie, for a given $(s, a) \in \Sc \times \Ac$ let $X_t$ be
a random variable distributed $\mathcal{N}(\mu_{X_t}, \sigma_{X_t}^2)$ where
\[
  \mu_{X_t} = \sum_{s^\prime \in \Sc_{l+1}} \Big(\alpha_l^t(s, a, s^\prime)
V_{l+1}^\star(s^\prime) / \sum_{x} \alpha_l^t(s, a, x) \Big)
=
\sum_{s^\prime \in \Sc_{l+1}} \Expect^t (P_l(s^\prime \mid s, a)) V_{l+1}^\star(s^\prime)
\]
due to the Dirichlet assumption~\ref{ass1}. Due to
assumption~\ref{ass1} we know that $\mathrm{Span}(V_{l}^\star(s)) \leq
L - l$, so we choose $\sigma_{X_t}^2 = (L-l)^2 / (n_l^t(s, a) \vee 1)$.  Let $\Fc^{V}_t =
\Fc_t \cup \sigma(V^\star)$ denote the union of $\Fc_t$ and the sigma-algebra
generated by $V^\star$. Applying
lemma~\ref{l-ian} and the tower property of conditional expectation we have that
for $\beta \geq 0$
\begin{align}
\begin{split}
\label{e-ian-applied}
  \Expect^t \exp\Big( \beta \sum_{s^\prime \in \Sc_{l+1}} P_l(s^\prime \mid s, a) V_{l+1}^\star(s^\prime)\Big)
  &=\Expect_{V_{l+1}^\star}\bigg(\Expect_{P}\Big(\exp \beta \Big(\sum_{s^\prime \in \Sc_{l+1}}
P_l(s^\prime \mid s, a) V_{l+1}^\star(s^\prime)\Big) \big| \Fc^{V}_t\Big)\Big| \Fc_t\bigg)\\
&\leq \Expect_{V_{l+1}^\star}\left(\Expect_{X_t}(\exp \beta X_t | \Fc^{V}_t)\Big| \Fc_t\right)\\
&=\Expect_{V_{l+1}^\star}\left( \exp (\mu_{X_t} \beta + \sigma_{X_t}^2 \beta^2 / 2)\Big| \Fc_t\right)\\
  &=\Expect_{V_{l+1}^\star}^t \exp \Big( \beta \sum_{s^\prime \in \Sc_{l+1}} \Expect^t P_l(s^\prime
\mid s, a) V_{l+1}^\star(s^\prime) + \sigma_{X_t}^2\beta^2/2\Big),
\end{split}
\end{align}
the first equality is the tower property of conditional expectation, the
inequality comes from the fact that $P_l(s^\prime \mid s, a)$ is conditionally
independent of $V_{l+1}^\star(s^\prime)$ and applying lemma~\ref{l-ian}, the next
equality is applying the moment generating function for the Gaussian
distribution and the final equality is substituting in for $\mu_{X_t}$.
Now applying this result to the last term in \eqref{e-app-partial}
\begin{align*}
  &\log \Expect^t\exp\Big(\beta \sum_{s^\prime \in \Sc_{l+1}}P_l(s^\prime \mid s, a)
    V_{l+1}^\star(s^\prime)\Big) \\
    &\qquad\overset{\scriptscriptstyle (a)}{\leq}\log \Expect^t_{V_{l+1}^\star} \exp
  \Big( \beta \sum_{s^\prime \in \Sc_{l+1}} \Expect^t P_l(s^\prime \mid s, a) V_{l+1}^\star(s^\prime) +\sigma_{X_t}^2\beta^2/2 \Big)\\
    &\qquad\overset{\scriptscriptstyle (b)}{=} \log \Expect^t_{Q_{l+1}^\star} \exp \Big(\beta \sum_{s^\prime \in \Sc_{l+1}} \Expect^t P_l(s^\prime \mid s, a)
    \max_{a^\prime}Q_{l+1}^\star(s^\prime, a^\prime)\Big) + \sigma_{X_t}^2\beta^2/2 \\
    &\qquad\overset{\scriptscriptstyle (c)}{\leq} \sum_{s^\prime \in \Sc_{l+1}} \Expect^t P_l(s^\prime \mid s, a)
    \log  \Expect^t_{Q_{l+1}^\star}\exp \Big(\beta
    \max_{a^\prime}Q_{l+1}^\star(s^\prime, a^\prime)\Big) +  \sigma_{X_t}^2\beta^2/2\\
    &\qquad\overset{\scriptscriptstyle (d)}{\leq} \sum_{s^\prime \in \Sc_{l+1}} \Expect^t P_l(s^\prime \mid s, a) \log
    \sum_{a^\prime \in \Ac}\exp G_{l+1}^{Q | t}(s^\prime, a^\prime)(\beta) +\frac{\beta^2
(L-l)^2}{2 (n_l^t(s, a) \vee 1)}
\end{align*}
where (a) follows from Eq.~\eqref{e-ian-applied}) and the fact that log is
increasing, (b) is replacing $V^\star$ with $Q^\star$, (c) uses
Jensen's inequality and the fact that $\log \Expect \exp(\cdot)$ is convex, and (d)
follows by substituting in for $\sigma_{X_t}$ and since the max of a collection
of positive numbers is less than the sum.
Combining this and \eqref{e-app-partial} the inequality immediately follows.
\end{proof}
\subsection{Proof of lemma~\ref{l-delta_unroll}}
\begin{lemma}
\label{l-delta_unroll}
Following the policy induced by expected occupancy measure $\lambda_l^t \in
[0,1]^{|\Sc_l| \times A}$, $l=1, \ldots,L$, and the temperature schedule $\tau_t$ in
\eqref{e-betat} we have
\[
\Expect \sum_{t=1}^N \epregret^t(\tau_t, \lambda^t) \leq 2 \sqrt{(\sigma^2 +
L^2) S A T \log A (1+\log T/L)}.
\]
\end{lemma}
\begin{proof}
Starting from the definition of $\epregret$
\begin{align*}
\epregret^t(\tau_t, \lambda^t) &= \sum_{l=1}^L \sum_{(s, a) \in \Sc_l \times \Ac}\lambda_l^t(s, a)\left( \tau_t H\left( \frac{\lambda_l^t(s)}{\sum_b \lambda_l^t(s, b)} \right) +  \delta^t(s, a, \tau_t) \right)\\
 &\leq \tau_t L \log A + \tau^{-1}_t \sum_{l=1}^L \sum_{(s, a) \in \Sc_l \times \Ac}\lambda^t(s, a)
  \frac{(\sigma^2+L^2)}{2 (n_l^t(s, a) \vee 1)}
\end{align*}
which comes from the sub-Gaussian assumption on $G_l^{\mu|t}$ and the fact that
entropy satisfies $H(\pi(\lambda_s)) \leq \log A$ for all $s$. These two terms summed
up to $\episodes$ determine our regret bound, and we shall bound each one
independently. To bound the first term:
\begin{align*}
  L\log A\sum_{t=1}^\episodes \tau_t
  &\leq (1/2)\sqrt{(\sigma^2+ L^2)L S A \log A(1+\log T/L)} \sum_{t=1}^N 1/\sqrt{t}\\
&\leq \sqrt{(\sigma^2+ L^2) S A T \log A(1+\log T/L)},
\end{align*}
since $\sum_{t=1}^N 1/\sqrt{t} \leq \int_{t=0}^N 1/\sqrt{t} = 2 \sqrt{N}$, and
recall that $N = \ceil{T/L}$. For simplicity we shall take $T = NL$, \ie, we
are measuring regret at episode boundaries; this only changes whether or not
there is a small fractional episode term in the regret bound or not.

To bound the second term we shall use the pigeonhole principle
lemma~\ref{l-pigeonhole}, which requires knowledge of the process that generates
the counts at each timestep, which is access to the \emph{true} occupancy
measure in our case.  The quantity $\lambda^t$ is not the true occupancy measure
at time $t$, which we shall denote by $\nu^t$, since that depends on $P$ which
we don't have access to (we only have a posterior distribution over it). However
it is the \emph{expected} occupancy measure conditioned on $\Fc_t$, \ie,
$\lambda^t = \Expect^t \nu^t$, which is easily seen by starting from $\lambda_1^t(s, a) =
\pi_1^t(s, a) \rho(s) = \nu_1^t(s, a)$, and then inductively using:
\begin{align*}
  \Expect^t \nu_{l+1}^t(s^\prime, a^\prime) &=\Expect^t\Big(\pi_{l+1}^t(s^\prime,
  a^\prime)\sum_{(s, a) \in \Sc_l \times \Ac}  P_l(s^\prime \mid s, a) \nu_l^t(s, a)\Big)\\
  &= \pi_{l+1}^t(s^\prime, a^\prime)\sum_{(s, a) \in \Sc_l \times \Ac} \Expect^t (P_l(s^\prime \mid s, a))
\Expect^t \nu_l^t(s, a)\\
  &= \pi_{l+1}^t(s^\prime, a^\prime)\sum_{(s, a) \in \Sc_l \times \Ac} \Expect^t (P_l(s^\prime \mid
s, a)) \lambda_l^t(s, a)\\
  &= \lambda_{l+1}^t(s^\prime, a^\prime)
\end{align*}
for $l=1,\ldots, L$,
where we used the fact that $\pi^t$ is $\Fc_t$-measurable and the fact that
$\nu_l(s, a)$
is independent of downstream quantities.
Now applying lemma~\ref{l-pigeonhole}
\begin{align*}
\Expect\sum_{t=1}^{\episodes}\sum_{(s, a) \in \Sc_l \times \Ac}\frac{\lambda^t_l(s,
a)}{n_l^t(s, a) +1}
  &= \Expect\sum_{t=1}^{\episodes}\Expect^t\left(\sum_{(s, a) \in \Sc_l \times \Ac}\frac{\lambda^t_l(s,
a)}{n_l^t(s, a)+1}\right)\\
  &= \Expect\left(\sum_{t=1}^{\episodes}\sum_{(s, a) \in \Sc_l \times \Ac}\frac{\nu^t_l(s,
a)}{n_l^t(s, a)+1}\right)\\
  &\leq A |\Sc_l| (1 + \log \episodes),
\end{align*}
which follows from the tower property of conditional expectation and since the
counts at time $t$ are $\Fc_t$-measurable.  From Eq.~\eqref{e-betat} we
know that sequence $\tau^{-1}_t$ is increasing, so we can bound the second term
as
\begin{align*}
  \Expect \sum_{t=1}^\episodes \tau^{-1}_t \sum_{l=1}^L \sum_{(s, a) \in \Sc_l \times \Ac}
  \frac{\lambda_l^t(s, a)(\sigma^2+L^2)}{2 (n_l^t(s, a)+1)}
  &\leq (1/2)(\sigma^2 + L^2) \tau^{-1}_{\episodes} \Expect
  \sum_{l=1}^L\left(\sum_{t=1}^{\episodes}\sum_{(s,
  a) \in \Sc_l \times \Ac}\frac{\lambda_l^t(s, a)}{n_l^t(s, a)+1}\right)\\
  &\leq (1/2)(\sigma^2 + L^2) \tau^{-1}_{\episodes}
  \sum_{l=1}^L A |\Sc_l|  (1 + \log \episodes)\\
  &= (1/2)(\sigma^2 + L^2) \tau^{-1}_{\episodes}
   S A (1 + \log \episodes)\\
  &= \sqrt{(\sigma^2 + L^2) SAT \log A (1+\log T/L)},
\end{align*}
since $\sum_{l=1}^L |\Sc_l| = |\Sc| = S$ and using $\episodes = \ceil{T/L}$. Combining these two bounds we get our result.
\end{proof}

\subsection{Proof of maximal inequality lemma~\ref{l-max-ineq}}
\begin{lemma}
\label{l-max-ineq}
Let $X_i:\Omega \rightarrow \reals$, $i=1,\ldots,n$ be random variables
with cumulant generating functions $G^{X_i}:\reals
\rightarrow \reals$, then for any $\tau \geq 0$
\begin{equation}
\Expect \max_i X_i \leq \tau \log \sum_{i=1}^n \exp G^{X_i}(1/\tau).
\end{equation}
\end{lemma}
\begin{proof}
Using Jensen's inequality
\begin{align}
\begin{split}
\Expect \max_i X_i &= \tau \log \exp (\Expect \max_i X_i /\tau)\\
&\leq \tau\log \Expect \max_i (\exp X_i /\tau) \\
&\leq \tau\log \sum_{i=1}^n \Expect \exp X_i /\tau\\
&= \tau \log \sum_{i=1}^n \exp G^{X_i}(1/\tau),
\end{split}
\end{align}
where the inequality comes from the fact that the max over a collection of
nonnegative values is less than the sum.
\end{proof}

\subsection{Derivation of dual to problem \eqref{e-cvx-opt-prob}}
Here we shall the derive the dual problem to the convex optimization
problem \eqref{e-cvx-opt-prob}, which will be necessary to prove
a regret bound for the case where we choose $\tau^\star_t$ as the temperature
parameter. Recall that the primal problem is
\begin{align*}
\mbox{minimize} & \quad \Expect_{s \sim \rho} (\tau \lse{a \in \Ac} (K_1(s, a) / \tau)) \\
\mbox{subject to} & \quad K_l \geq \Bc^t_l(\tau, K_{l+1}),\quad l = 1, \ldots, L, \\
& \quad K_{L+1} \equiv 0,
\end{align*}
in variables $\tau \geq 0$ and $K \in \reals^{S \times A}$.
We shall repeatedly use the variational representation of log-sum-exp terms as
in Eq.~\eqref{e-variational}. We introduce dual variable $\lambda \geq 0$ for each
of the $L$ Bellman inequality constraints which yields Lagrangian
\[
\sum_{s \in \Sc_1}\rho(s)\sum_{a \in \Ac} (\pi_1(s, a) K_1(s, a) + \tau H(\pi_1(s)))
+\sum_{l=1}^L \lambda^T(\Bc_l^t(\tau, K_{l+1}) - K_l).
\]
For each of the $L$ constraint terms we can expand the $\Bc_l$ operator and
use the variational representation for log-sum-exp to get
\[
\sum_{(s, a) \in \Sc_l \times \Ac} \lambda_l(s, a)\bigg(\tau \tilde G_l^{\mu|t}(s, a, 1/\tau) +
\sum_{s^\prime \in \Sc_{l+1}} \Expect^t P_l(s^\prime \mid s, a) \Big(\sum_{a^\prime \in \Ac}
\pi_{l+1}(s^\prime, a^\prime) K_{l+1}(s^\prime, a^\prime) + \tau
H(\pi_{l+1}(s^\prime))\Big) - K_l(s, a)\bigg).
\]
At this point the Lagrangian can be expressed:
\begin{align*}
\Lc(\tau, K, \lambda, \pi) &= \sum_{s \in \Sc_1}\rho(s)\big(\sum_{a \in \Ac} (\pi_1(s, a) K_1(s,
a)) + \tau H(\pi_1(s))\big) + \sum_{l=1}^L \sum_{(s, a) \in \Sc_l \times \Ac} \lambda_l(s, a)\bigg(\tau \tilde G_l^{\mu|t}(s, a, 1/\tau) +\\
&\qquad +
  \sum_{s^\prime \in \Sc_{l+1}} \Expect^t P_l(s^\prime \mid s, a) \Big(\sum_{a^\prime \in \Ac}
\pi_{l+1}(s^\prime, a^\prime) K_{l+1}(s^\prime, a^\prime) + \tau
H(\pi_{l+1}(s^\prime))\Big) - K_l(s, a)\bigg).
\end{align*}
To obtain the dual we must minimize over $\tau$ and $K$. First, minimizing over
$K_1(s, a)$ yields
\[
\rho(s) \pi_1(s, a) = \lambda_1(s, a)
\]
and note that since $\pi_1(s)$ is a probability distribution it implies
that
\[
\sum_{a_1 \in \Ac}\lambda_1(s, a) = \rho(s)
\]
for each $s \in \Sc_1$. Similarly we minimize over each
$K_{l+1}(s^\prime, a^\prime)$ for $l = 1, \ldots, L$ yielding
\[
\lambda_{l+1}(s^\prime, a^\prime) = \pi_{l+1}(s^\prime, a^\prime)\sum_{(s, a) \in \Sc_l \times \Ac}
\Expect^t P_l(s^\prime \mid s, a)\lambda_l(s, a).
\]
which again implies
\[
\sum_{a^\prime \in \Ac} \lambda_{l+1}(s^\prime, a^\prime) = \sum_{(s, a) \in \Sc_l \times \Ac} \Expect^t
P_l(s^\prime \mid s, a)\lambda_l(s, a).
\]
What remains of the Lagrangian is
\[
\sum_{l=1}^L \sum_{(s, a) \in \Sc_l \times \Ac} \lambda_l(s, a) \left(\tau \tilde G_l^{\mu|t}(s, a,
1/\tau) + \tau H(\pi_l(s, \cdot))\right)
\]
which, using the definition of $\delta$ in Eq.~\eqref{e-little-delta} can be rewritten
\[
\sum_{l=1}^L \sum_{(s, a) \in \Sc_l \times \Ac} \lambda_l(s, a)\Expect^t \mu_l(s, a) +
\min_{\tau \geq 0} \sum_{l=1}^L \sum_{(s, a) \in \Sc_l \times \Ac}\lambda_l^t(s, a)\left( \tau_t
H\left( \frac{\lambda_l(s)}{\sum_b \lambda_l(s,b)}\right) +  \delta_l^t(s, a, \tau_t) \right).
\]
Finally, using the definition of $\epregret$ in \eqref{e-epregret-def} we obtain:
\begin{align}
\begin{split}
\label{e-cvx-dual}
  \mbox{maximize} & \quad
  \sum_{l=1}^L  \sum_{(s, a) \in \Sc_l \times \Ac} \lambda_l(s, a) \Expect^t \mu_l(s, a)
  + \min_{\tau \geq 0} \epregret^t(\tau, \lambda) \\
  \mbox{subject to}
  & \quad \sum_{a^\prime \in \Ac}\lambda_{l+1}(s^\prime, a^\prime) =
  \sum_{(s, a) \in \Sc_l \times \Ac} \Expect^t (P_l(s^\prime \mid s, a)) \lambda_l(s, a),
  \quad s^\prime \in \Sc_{l+1},\ l = 1, \ldots, L \\
  & \quad \sum_{a_1}\lambda_1(s, a) = \rho(s), \quad s \in \Sc_1 \\
  & \quad \lambda \geq 0.
\end{split}
\end{align}

\subsection{Proof of Lemma \ref{l-cvx-opt-regret}}
\begin{lemma}
\label{l-cvx-opt-regret}
Assuming strong duality holds for problem \eqref{e-cvx-opt-prob},
and denote the primal optimum at time $t$ by $(\tau_t^\star, K_l^{t\star})$ then the policy given by
\[
  \pi_l^t(s, a) \propto \exp(K_l^{t\star}(s, a) / \tau_t^\star) %
\]
satisfies the Bayes regret bound given in Theorem \ref{t-kl_mdp}.
\end{lemma}
\begin{proof}
The dual problem to \eqref{e-cvx-opt-prob} is derived above as Eq.~\eqref{e-cvx-dual}.
Denote by $\Lc^t$ the (partial) Lagrangian at time $t$:
\[
\Lc^t(\tau, \lambda) =  \sum_{l=1}^L  \sum_{(s, a) \in \Sc_l \times \Ac} \lambda_l(s, a) \Expect^t \mu_l(s, a)
  + \epregret^t(\tau, \lambda).
\]
Denote by $\lambda^{t\star}$ the dual optimal variables at time $t$.
Note that the value $\Lc^t(\tau_t^\star, \lambda^{t\star}_l)$ provides an upper bound on
$\Expect^t \max_{a} Q^\star_1(s,a)$ due to strong duality. Furthermore
we have that
\[
  \sum_{l=1}^L \sum_{(s, a) \in \Sc_l \times \Ac} \lambda^{t\star}_l(s, a) \Expect^t \mu_l(s,a) =
\Expect_{s \sim \rho}\Expect^t V_1^{\pi^t}(s),
\]
and so using \eqref{e-bayesregret1} we can bound the regret of following the policy induced by $\lambda^{t\star}$ using
\begin{equation}
\label{e-br-opt}
\bayesregret(T) \leq \Expect \sum_{t=1}^N \Big( \Lc^t(\tau^\star_t, \lambda^{t\star}) -
  \sum_{l=1}^L \sum_{(s, a) \in \Sc_l \times \Ac} \lambda^{t\star}_l(s, a) \Expect^t \mu_l(s,a) \Big) =  \Expect \sum_{t=1}^N \epregret^t(\tau^\star_t, \lambda^{t\star}).
\end{equation}
Strong duality implies that the Lagrangian has a saddle-point at $\tau_t^\star, \lambda^{t\star}$
\[
\Lc^t(\tau_t^\star, \lambda) \leq \Lc^t(\tau_t^\star, \lambda^{t\star}) \leq \Lc^t(\tau, \lambda^{t\star})
\]
for all $\tau \geq 0$ and feasible $\lambda$, which immediately implies the
following
\begin{equation}
\label{e-tau-isopt}
\epregret^t(\tau_t^\star, \lambda^{t\star}) = \min_{\tau \geq 0} \epregret^t(\tau, \lambda^{t\star}).
\end{equation}
Now let $\tau_t$ be the temperature schedule in \eqref{e-betat}, we have
\[
\bayesregret(T) \leq \Expect \sum_{t=1}^N \epregret^t(\tau_t^\star,
\lambda^{t\star}) = \Expect \sum_{t=1}^N \min_{\tau \geq 0} \epregret^t(\tau,
\lambda^{t\star}) \leq \Expect \sum_{t=1}^N \epregret^t(\tau_t,
\lambda^{t\star}) \leq \tilde O(L \sqrt{S A T}),
\]
where the last inequality comes from applying lemma~\ref{l-delta_unroll}, which
holds for any occupancy measure when the agent is following the corresponding
policy.
\end{proof}

\subsection{Proof of pigeonhole principle lemma \ref{l-pigeonhole}}
\begin{lemma}
\label{l-pigeonhole}
Consider a process that at each time $t$ selects a single index $a_t$
from $\{1, \ldots, m\}$ with probability $p^t_{a_t}$. Let $n^t_{i}$ denote the
count of the number of times index $i$ has been selected up to time $t$. Then
\[
  \sum_{t=1}^N \sum_{i=1}^{m} p^t_{i} / (n^t_{i} \vee 1) \leq m (1 + \log N).
\]
\end{lemma}
\begin{proof}
This follows from a straightforward application of the pigeonhole principle,
\begin{align*}
  \sum_{t=1}^N \sum_{i=1}^{m} p_i^t / (n_i^t \vee 1)
  &= \sum_{t=1}^N \Expect_{a_t \sim p^t} (n^t_{a_t} \vee 1)^{-1}  \\
&= \Expect_{a_0 \sim p^0, \ldots,
  a_N \sim p^t} \sum_{t=1}^N (n^t_{a_t} \vee 1)^{-1} \\
&= \Expect_{a_0 \sim p^0, \ldots,
a_N \sim p^t} \sum_{i=1}^{m} \sum_{t=1}^{n_i^N \vee 1} 1/t \\
&\leq \sum_{i=1}^{m} \sum_{t=1}^{N} 1/t  \\
  &\leq m (1 + \log N),
\end{align*}
  where the last inequality follows since $\sum_{t=1}^{N} 1/t \leq 1 +
  \int_{t=1}^{N} 1/t = 1 + \log N$.
\end{proof}

\section{Compute requirements}
All experiments were run on a single 2017 MacBook Pro.

\end{document}